\documentclass[letterpaper]{article} 
\usepackage{aaai24}  
\usepackage{times}  
\usepackage{helvet}  
\usepackage{courier}  
\usepackage[hyphens]{url}  
\usepackage{graphicx} 
\usepackage{natbib}  
\usepackage{caption} 
\usepackage{algorithm}
\usepackage{algorithmic}

\newtheorem{theorem}{Theorem}
\newtheorem{proof}{Proof}
\usepackage{color}
\usepackage{times}
\usepackage{epsfig}
\usepackage{subfigure}

\usepackage{amsmath}
\usepackage{amssymb}
\usepackage{latexsym}
\usepackage{svg}
\usepackage{multibib}

\usepackage{multirow}
\usepackage{verbatim}

\usepackage{hyperref}

\usepackage{newfloat}
\usepackage{listings}
\DeclareCaptionStyle{ruled}{labelfont=normalfont,labelsep=colon,strut=off} 
\floatstyle{ruled}
\newfloat{listing}{tb}{lst}{}
\floatname{listing}{Listing}
\title{Contrastive Continual Learning with Importance Sampling and \\ Prototype-Instance Relation Distillation}
\author{
    Jiyong Li\textsuperscript{\rm 1, \rm 2}, Dilshod Azizov\textsuperscript{\rm 3}, Yang Li\textsuperscript{\rm 4, \rm 5}, Shangsong Liang\textsuperscript{\rm 1, \rm 2, \rm 3}\thanks{Corresponding author.}
}
\affiliations{
    \textsuperscript{\rm 1}School of Computer Science and Engineering, Sun Yat-sen University, China\\

    \textsuperscript{\rm 2}Guangdong Key Laboratory of Big Data Analysis and Processing, Guangzhou, China\\
    \textsuperscript{\rm 3}Department of Machine Learning, Mohamed bin Zayed University of Artificial Intelligence, United Arab Emirates\\
    \textsuperscript{\rm 4}AI Thrust, Information Hub, The Hong Kong University of Science and Technology (Guangzhou), China\\
    \textsuperscript{\rm 5}Department of CSE, The Hong Kong University of Science and Technology, China\\
    lijy373@mail2.sysu.edu.cn, dilshod.azizov@mbzuai.ac.ae, liyang259@mail2.sysu.edu.cn, \\
    liangshangsong@gmail.com
}

\begin{document}

\maketitle

\begin{abstract}
Recently, because of the high-quality representations of contrastive learning methods, rehearsal-based contrastive continual learning has been proposed to explore how to continually learn transferable representation embeddings to avoid the catastrophic forgetting issue in traditional continual settings. Based on this framework, we propose \emph{\textbf{C}ontrastive \textbf{C}ontinual \textbf{L}earning via \textbf{I}mportance \textbf{S}ampling (CCLIS)} to preserve knowledge by recovering previous data distributions with a new strategy for \emph{\textbf{R}eplay \textbf{B}uffer \textbf{S}election (RBS)}, which minimize estimated variance to save hard negative samples for representation learning with high quality.  Furthermore, we present the \emph{\textbf{P}rototype-instance \textbf{R}elation \textbf{D}istillation (PRD)} loss, a technique designed to maintain the relationship between prototypes and sample representations using a self-distillation process. Experiments on standard continual learning benchmarks reveal that our method notably outperforms existing baselines in terms of knowledge preservation and thereby effectively counteracts catastrophic forgetting in online contexts. The code is available at https://github.com/lijy373/CCLIS.
\end{abstract}

\section{Introduction}

Deep neural networks are widely recognized to commonly face challenges in preserving performance in learned tasks after being trained on a new one, which is called Catastrophic Forgetting~\cite{mccloskey1989catastrophic}.
Continual learning has been proposed to overcome Catastrophic Forgetting and transfer knowledge forward and backward during training for a stream of data, with only a small portion of the samples available at once. Recent research has integrated this scheme into continual learning due to the capability of contrastive learning to secure high-quality representations~\cite{chen2020simple}. A prime example of this integration is contrastive continual learning~\cite{cha2021co2l}, which leverages the advantages of contrastive representations. Furthermore, we have identified two primary areas that remain underexplored in contrastive continual learning. First, the settings of continual learning inherently mean that only limited samples from previous tasks are retained and used alongside current tasks in contrastive continual models. This introduces a bias between contrastive representations trained in on-line and off-line settings, significantly contributing to Catastrophic Forgetting. Second, while the importance of hard negative samples (i.e., data points that are difficult to distinguish from anchors of different classes) for contrastive representation is recognized~\cite{robinson2020contrastive}, limited research has been done on the optimal selection and preservation of these samples in the context of contrastive continual learning, as observed in previous research~\cite{cha2021co2l}. In this paper, we address two questions: (1) \emph{How can we more effectively recapture the data distributions of earlier tasks using buffered samples?} and (2) \emph{How can we improve the selection and retention of important samples?}

Consequently, we propose \emph{ CCLIS}, a contrastive continual learning algorithm that can alleviate Catastrophic Forgetting and select hard negative samples. We set out to make theoretical and empirical contributions as follows: 
\begin{enumerate}
\item Based on a prototype-based supervised contrastive continual learning framework, we estimate previous task distributions via importance sampling with a weighted buffer during training, which can eliminate the bias between contrastive representations trained online and offline to overcome Catastrophic Forgetting.

\item Based on (1), we introduce a methodology to compute importance weights within a contrastive continual learning framework accompanied by robust theoretical justifications. In addition, we examine the interplay between importance sampling and hard negative selection. In particular, we observe that samples with higher scores, as determined by our method, are more likely to be retained as hard negative samples.

\item We propose the \emph{PRD} loss to preserve the learned relationship between prototype and instance, which helps to use the importance sampling method to recover the distributions of tasks and maintain knowledge of previous tasks. 

\item Finally, experiments demonstrate that our method outperforms most of the state-of-the-art non-stationary task distributions from three benchmark datasets. Empirically, our algorithms can recover the data distributions of previous tasks as much as possible and store hard negative samples to enhance the performance of contrastive continual learning, which helps mitigate Catastrophic Forgetting.

\end{enumerate}

\begin{figure*}[!t]
\centering
\subfigure[Contrastive Learning via Importance Sampling]{
\label{fig:CCLIS}
\includegraphics[width=.5\textwidth]{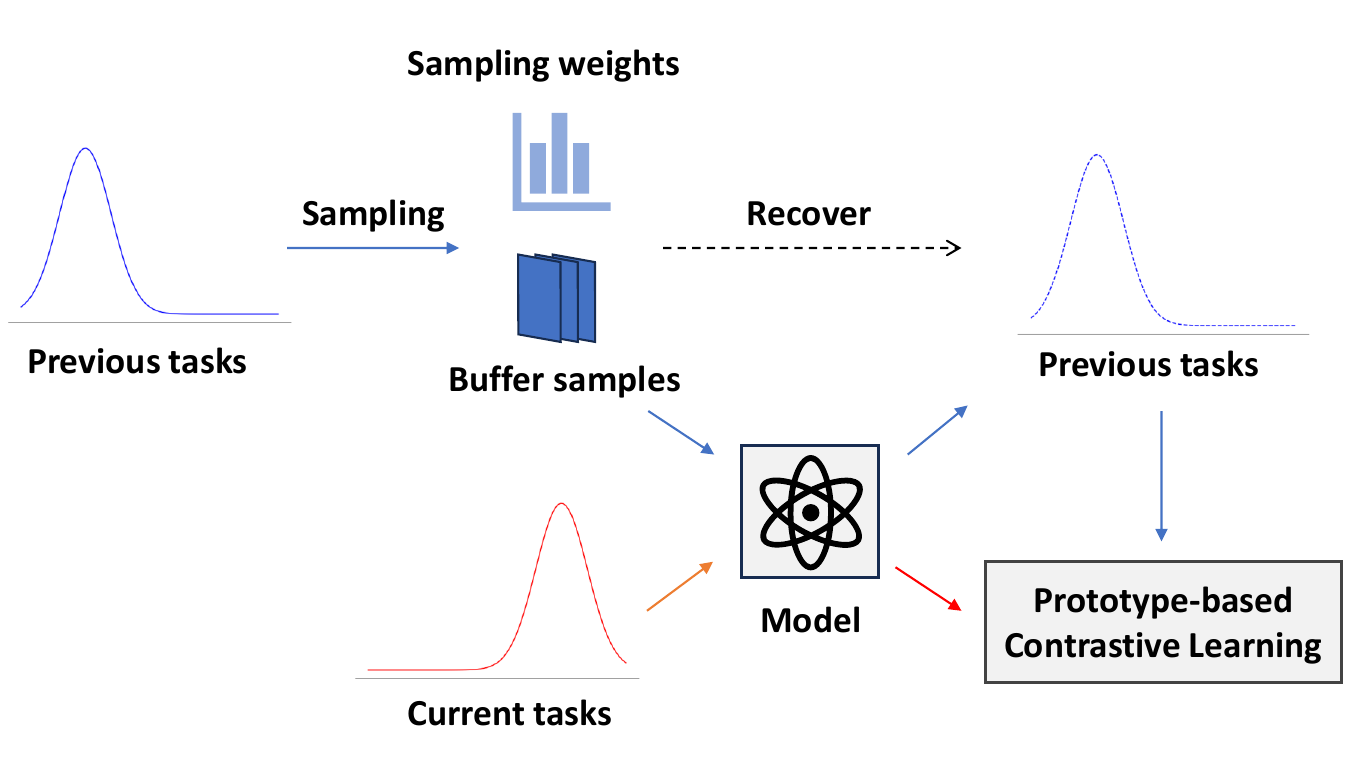}}
\subfigure[\footnotesize{Prototype-instance} Relation Distillation Loss]{\label{fig:PRD}
\includegraphics[width=.43\textwidth,]{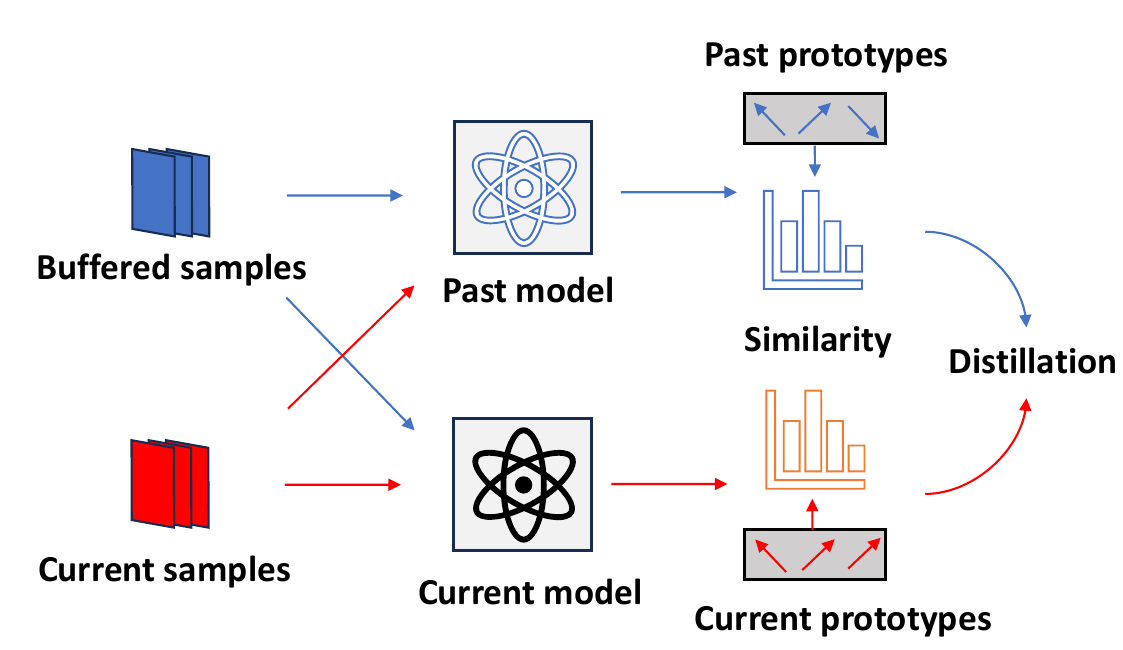}}

\caption{Illustration of Contrastive Learning via Importance Sampling and PRD Loss. (a) When new tasks are introduced, buffer samples are drawn with specific sampling weights. By using the Importance Sampling technique, we approximately recover the data distributions of previous tasks and apply prototype-based contrastive learning on previous and current data to have high-quality features.  (b) Given samples of a mini-batch, the PRD loss is designed to distill the relation between prototypes and instances from the previous model to the current one. We minimize the cross-entropy of prototype-instance similarity from the current and previous models with frozen parameters, which are computed with dot products of normalized embeddings.}
\label{fig:Overview}

\end{figure*}

\section{Related Work}

\textbf{Contrastive Continual learning.} 
Continual learning aims to extend the previous knowledge to new task adaptation and overcome Catastrophic Forgetting with restricted buffer and computing resources for the data stream, which has been widely used in Computer Vision, Reinforcement Learning etc~\cite{wang2023comprehensive}. Continual learning algorithms are divided into three categories~\cite{de2021continual}: Replay methods~\cite{riemer2018learning,lopez2017gradient,aljundi2019gradient}, Regularization-based methods~\cite{li2017learning,zhang2021variational} and Parameter isolation methods~\cite{rusu2016progressive}. Recently, replay methods have shown impressive performance, where the idea of extending contrastive learning~\cite{jaiswal2020survey, khosla2020supervised} schemes to continual learning settings is noteworthy~\cite{cha2021co2l, liang2019collaborative, liang2018dynamic, liang2021cross, 
de2021continual}. However, due to the lack of previous task samples caused by continual learning, there is a gap in representations between continual contrastive learning and off-line authentic contrastive learning. To fill this gap, recent work proposes continual contrastive learning \cite{fini2022self} and contrastive continual learning \cite{cha2021co2l}. The first aims to explore the performance of contrastive learning in a continuous learning setting, which is a different research goal from this paper. The latter adopts contrastive learning as a method to avoid suffering from Catastrophic Forgetting, which is the topic our work mainly focuses on. In this study, with the help of contrastive continual learning, we use the importance sampling method to recover the data distribution of previous tasks to alleviate Catastrophic Forgetting to some extent.

\textbf{Coreset selection for Replay-based Continual Learning.} Our algorithm focuses primarily on the replay method. Coreset selection for replay-based continual learning has been studied, and some buffer selection criteria have been proposed, for example, herding-based exemplar selection \cite{rebuffi2017icarl}, gradient-based method~\cite{aljundi2019gradient}, active learning~\cite{ayub2022few}, mini-batch similarity, sample diversity and coreset affinity \cite{yoon2021online} etc. Recent work \cite{tiwari2022gcr} proposes weighted buffer selection, which can increase the capacity to preserve previous knowledge.
Although coreset selection has been proposed in continual learning and contrastive learning~\cite{robinson2020contrastive}, how to select and preserve important samples to help contrastive continual learning has not yet been studied.
With 
contrastive continual learning,  
we apply the importance sampling method to weighted buffer, which assists mine hard negatives, estimate data distributions of previous tasks and overcome Catastrophic Forgetting issue.

\textbf{Knowledge Distillation.} As a vital method to transfer knowledge from teacher models to student models, knowledge distillation~\cite{hinton2015distilling} has been widely applied to continual learning in order to alleviate Catastrophic Forgetting by distilling past features into current models. Inspired by~\cite{cha2021co2l, fang2021seed, asadi2023prototype}, we propose to use knowledge distillation to keep the relation between prototypes and instances stable, which helps to apply the importance sampling strategy to better recover previous task distributions.

\section{Background}
\subsection{Problem Setup: Continual Learning}
\label{sec:continual}

We aim at supervised 
continual learning: given sequentially arriving tasks with timesteps $t \in \{1,2,\dots, T\}$, which are assumed to be sampled from a non-stationary distribution. In each task $t$, there are $N_t$ input-label pairs drawn from the data distributions of the specific task, i.e., ${(x_i, y_i)}_{i=1}^{N_t}\sim \mathcal{D}_t$, where $x_i$ denotes the input and $y_i$ denotes the responding class label. At each time step $t$, the model receives current task samples $D_t$ separate from past task data $D_{1:t-1}$ due to the non-stationary distribution from which they are sampled.

We focus on two primary tasks in the setting of continual learning~\cite{van2019three}:  
\emph{Class-incremental learning (Class-IL) and Task-incremental learning (Task-IL).} For the first, given test samples, the model predicts the responding labels without task information during evaluation; that is, the goal is to obtain the model $\phi_\theta(x)$ by optimizing the following objective:

\begin{equation}
\small
     \min\sum_{t=1}^T\mathbb{E}_{(x,y)\sim D_t}[l(y, \phi_\theta(x))],
    \label{Class-IL}
\end{equation}

\noindent where $l$ denotes the loss function such as the softmax function and $\theta$ denotes the model parameters. For the latter, the model is trained on sequentially arriving tasks with clear task boundaries and evaluated with task information:
\begin{equation}
\small
     \min\sum_{t=1}^T\mathbb{E}_{(x,y)\sim D_t}[l(y,\phi_\theta(x,t))].
    \label{Task-IL}
\end{equation}

\subsection{Preliminaries: Contrastive Learning}

Given a training dataset $\mathcal{X}=\{x_1,x_2,\dots,x_N\}$ of $N$ samples, contrastive learning aims to obtain an embedding function $f_\theta$ mapping $\mathcal{X}$ to representation embeddings $\mathcal{Z} = \{z_1, z_2,\dots, z_N\}$, i.e., $z_i = f_\theta(x_i)$, which can be applied to downstream tasks. Instance-wise contrastive learning \cite{he2020momentum} augments training samples into two views to have $2N$ inputs and gains high-quality embedding space by minimizing a contrastive loss function such as InfoNCE \cite{oord2018representation}:
\begin{equation}
\small
    L_{\mathrm{InfoNCE}}(\theta) := \sum_{i=1}^{2N} -\log\frac{\exp(z_i\cdot z_i'/\tau)}{\sum_{j\neq i}\exp(z_i\cdot z_j/\tau)},
    \label{eq: InfoNCE}
\end{equation}
 where $z_i$ and $z_i'$ are positive embedding pairs of instance $x_i$ with regards to the two views, while $z_j$ includes one positive embedding $z_i'$ and $2N-2$ negative embeddings of other instances, and $\tau$ denotes temperature.

To better apply our proposed method, we replace $z_i$ with the prototype $c$ of the responding class in equation~\eqref{eq: InfoNCE} without amplification of samples into two views, which is different from the previous prototype-based contrastive learning framework~\cite{li2020prototypical, caron2020unsupervised}.

\subsection{Preliminaries: Importance Sampling}
As a significant method in machine learning, importance sampling~\cite{kong1994sequential} has been proposed to approximate expectations without drawing samples from target distributions. Formally, given the data distribution $\pi$ of sample $z$ and the proposal distribution $q$, we can estimate the expectation of function $f(z)$ on $\pi$ using the classical Monte Carlo Importance Sampling method:
\begin{equation}
\small
    \mathbb{E}_{z\sim \pi}f(z) = \mathbb{E}_{z\sim q}\frac{\pi(z)}{q(z)}f(z) \approx \frac{1}{L}\sum_{l=1}^L \frac{\pi(z_l)}{q(z_l)} f(z_l),
\end{equation}
where $L$ is the number of drawn samples. However, it is often the case that assuming $\pi = \hat{\pi}/Z_\pi$, the distribution $\pi$ can only be evaluated up to an incremental normalization constant $Z_\pi$. By estimating $Z_\pi$, the biased importance sampling strategy provides the following approximation of the expectation:
\begin{equation}
\small
    \mathbb{E}_{z\sim \pi}f(z) \approx \frac{1}{L}\sum_{l=1}^L \frac{\hat{\pi}(z_l)/q(z_l)}{\hat{Z}_\pi} f(z_l),
\end{equation}
where $\hat{Z}_\pi:=\frac{1}{L}\sum_{l=1}^L \hat{\pi}(z_l)/q(z_l)$ is the estimator of  $Z_\pi$.

\section{Contrastive Continual Learning via Importance Sampling}
\subsection{Overview of Our Model}
Based on previous contrastive continual learning~\cite{cha2021co2l}, we propose \emph{CCLIS} in order to overcome Catastrophic Forgetting. 
We propose a model structured in three distinct steps: First, we build the supervised version of prototype-based InfoNCE via importance sampling in the setting of the continual learning framework. Second, based on our modified version of contrastive loss, we propose a sampling method for \emph{RBS} with sufficient theoretical guarantees. Finally, we introduce a distillation loss in our model. By doing so, our model boasts several significant advantages: (1) With importance sampling under the prototype-based InfoNCE, we can recover distributions of previous tasks to overcome Catastrophic Forgetting. (2) With our \emph{RBS}, we can draw hard negative samples to preserve and eliminate the estimated variation in importance sampling. (3) By using the \emph{PRD} loss, we can maintain the relationship between prototypes and instances, improving the performance of the importance sampling method. In addition, it helps to distill and preserve knowledge from previous tasks. The overview of the proposed model is presented in Algorithm~\ref{alg1}.

\subsection{Prototype-based InfoNCE Loss via Importance Sampling}

In this section, our aim is to adopt a supervised version of the prototype-based InfoNCE to learn high-quality contrastive representation in a continual learning setting. However, due to the gap between offline and online scenarios, only a few samples of previous tasks are preserved in the replay buffer $M$, making it infeasible to optimize the loss function. To fill this gap, we aim to keep the performance of our model close to that trained on data drawn from the current task $D_t$ and the samples available at the moment $t-1$, that is, $R_{t-1}:=M_{t-1}\cup D_{t-1}$, as it is difficult to achieve performance learning samples from all tasks $D_{1:t}$.
Formally, for each task $t$, $\hat{C}_t$ and $Y_t$ are defined as the class set of $R_{t}$ and $D_t$, respectively. Assuming that dataset $\mathcal{D}$ has $N$ samples with $K$ classes, we denote prototypes $\{c_1,\dots,c_K\}$ as clustering centers of classes that are randomly initialized. By denoting $s_{ij}:=c_i\cdot z_j/\tau$ for convenience, we focus on the following objective:
\begin{equation}
    \small
    L:=\sum_{i\sim \hat{C}_{t-1} \cup Y_t}\sum_{j\sim S_i}-\log\frac{\exp(s_{ij})}{\sum_{k\sim R_{t-1}\cup D_t}\exp(s_{ik})}
   \label{Sample-ProtoInfoNCE}
\end{equation}
where $S_i:=\{j|y_j=i\; \text{and}\; (x_j,y_j)\in R_{t-1}\cup D_t\}$ denotes samples from specific class $i$ in currently available data. To analyze the contrastive loss, we deviate from it with respect to the model parameters and obtain the gradients of this loss function. 
For a specific prototype $c_i$ of the previous task, we denote $p_{ij}=\frac{\exp(s_{ij})}{\sum_{k\sim R_{t-1}\cup D_t}\exp(s_{ik})}$, then the gradient of the loss function~(\ref{Sample-ProtoInfoNCE}) has the following form if we focus on the sample $j$ drawn from $S_i$:
\begin{equation}
\small
\nabla_\theta L_{i,j}=-\nabla_\theta s_{ij}+\sum_{k\sim R_{t-1}\cup D_t} p_{ik}\nabla_\theta s_{ik}
\label{gradient}
\end{equation}
\noindent where $L_{i,j} = -\log\frac{\exp(s_{ij})}{\sum_{k\sim R_{t-1}\cup D_t}\exp(s_{ik})}$. Note that the gradient $\nabla_\theta L_{i,j}$ can be decomposed into two parts, including $-\nabla_\theta s_{ij}$ and $\sum_{k\sim R_{t-1}\cup D_t} p_{ik}\nabla_\theta s_{ik}$, while the latter part can be written as the expectation of the gradients $\mathbb{E}_{k\sim p_i}\nabla_\theta s_{ik}$. However, it is difficult to apply the \emph{classical Monte-Carlo} method to estimate expectations on various target distributions $p_i$. It is also difficult to directly adopt the \emph{classical Importance Sampling} method, since the partition function $\sum_{k\sim R_{t-1}\cup D_t}\exp(s_{ik})$ is intractable due to previous lost samples. To address the issues, we use the biased importance sampling method to approximate the intractable normalizing function and give a biased estimation of the gradients, while positive samples of specific class $i$ and samples from task $t$ cannot be estimated. Unlike the traditional sampling strategy, each instance $(x_j,y_j)$ is sampled from the discrete distribution of the specific class instead of the entire data distribution, avoiding the mixing of positive and negative samples. Formally, we denote $g^{(m)}=[g_{j}^{(m)}]_{j\sim S_m}$ as the proposal distribution of the specific class $m$, and $J_m$ represents the samples drawn from $g^{(m)}$. Accordingly, we get the biased estimate of the stochastic gradient as below using biased importance sampling and Monte-Carlo techniques:
\begin{align}
\small
\nabla_\theta L_{i,j}\approx &-\nabla_\theta s_{ij}+\sum_{k\sim J_i\cup D_t}\frac{p_{ik}}{W_i}\nabla_\theta s_{ik} \displaybreak[3] \nonumber\\
    &+\sum_{m\sim \hat{C}_{t-1}\backslash i}\frac{1}{|J_m|}\sum_{k\sim J_m}\frac{\omega_{ik}^{(m)}}{W_i}\nabla_\theta s_{ik}\,,\label{gradient_IS}
\end{align}
\noindent where $\omega_{ik}^{(m)}:= p_{ik}^{(m)}/{g_{k}^{(m)}}$ and $W_i := \sum_{k\sim J_i\cup D_t}p_{ik} + \sum_{m\sim \hat{C}_{t-1}\backslash i}\frac{1}{|J_m|}\sum_{k\sim J_m}\omega_{ik}^{(m)}$. Finally, our modified version of contrastive loss can be obtained as the sum of the antiderivative of the gradient in the specific task $t$:
\begin{equation}
\small
    L_{\mathrm{Sample-NCE}}(\theta;t) =\sum_{i\sim Y_t}\sum_{j\sim S_i} L_{i,j} + \sum_{i\sim \hat{C}_t} \sum_{j\sim J_i} \hat{L}_{i,j},
    \label{eq:total}
\end{equation}
where:
\begin{equation}
\small
    \hat{L}_{i,j}=-\log\frac{\exp(s_{ij})}{\sum\limits_{k\sim J_i\cup D_t}\exp(s_{ik})+\sum\limits_{m\sim \hat{C}_{t-1}\backslash i}\sum\limits_{k\sim J_{m}}\frac{\exp(s_{ik})}{g_{k}^{(m)}|J_m|}}.
    \label{contrastive_IS}
\end{equation}

The loss function can be applied directly to samples drawn from mini-batches. Now, we propose a method to recover data distributions from previous tasks using weighted samples in the replay buffer. For more in-depth technical details, please refer to Appendix A.

\begin{algorithm}[!t]
	
	\caption{\emph{\textbf{C}ontrastive \textbf{C}ontinual \textbf{L}earning via \textbf{I}mportance \textbf{S}ampling (CCLIS).}}
	\label{alg1}
	\begin{algorithmic}[1]
            \REQUIRE Dataset $\{D_t\}_{t=1}^T$, model $f_\theta$, buffer $M\gets\{\}$, hyper-parameter $\lambda$, learning rate $\eta$,
		\STATE Initialize model parameters $\theta$,
            \FOR{t=1,\dots,T}
            \FOR{batch $B_t\sim D_t$}
            \STATE $B_M\sim M$
		\STATE  $B \gets B_t\cup B_M$
            \STATE $L\gets L_{\mathrm{Sample-NCE}}(\theta; B)$ with Eq.(\ref{eq:total})
            \IF{t $>$ 1}
            \STATE $L\gets L + \lambda\cdot L_{PRD}(\theta; \theta_{prev}, B)$ with Eq.(\ref{PRD})
            \ENDIF
            \STATE $\theta \gets \theta - \eta\nabla_{\theta}L$ 
            \ENDFOR
            \STATE Calculate the proposal distribution $g$ with Eq.(\ref{eq:score})
            \STATE $M\gets SELECT(M\cup D_t,g)$ by weighted sampling without replacement 
            \STATE $\theta_{prev}\gets \theta$
            \ENDFOR
		
	\end{algorithmic}  
\end{algorithm}

\subsection{Replay Buffer Selection for Estimated Variance Minimization}

While using the modified contrastive learning objective can help recover the previous data distributions, one may still benefit from selecting a suitable proposal distribution, which is crucial to the effectiveness of the importance sampling method. 
The closer the proposal distribution to the target distribution, the better the model performance, and the gap between the proposal distribution and the target distribution would cause the estimated variance, leading to a drop in performance. However, it is not trivial to select a proposal distribution $g^{(m)}$ of a specific class $m$ in $\hat{C}_{t}$, because there are $n_t:=|\hat{C}_{t}|-1$ target distributions, except the distribution $p_{m}$ to which the proposal distribution must be close. To overcome this, we propose \emph{RBS} based on the estimated minimization of variance, which can better recover the previous data distributions.

For the set of a specific class $m$ sampled from $R_t$, the target distribution with the fixed prototype $c_i$ is $p_{i}$. Although samples from $D_t$ should be considered, we can approximately minimize the estimated variance caused by the gap between the proposal distribution $g^{(m)}$ and the target distributions $\hat{p}_{i}^{(m)}:=[\frac{\exp{s_{ij}}}{\sum_{j\sim R_{t-1}}\exp{s_{ij}}}]_{j\sim S_m}$ for all $i\in \hat{C}_{t}\backslash m$. We prove in Appendix A that we can minimize an upper bound on the estimated variance by optimizing the mean of Kullback–Leibler (KL) divergences of the proposal distribution from the target distributions:
\begin{equation}
\small
    \min\frac{1}{n_t}\sum_{i=1}^{n_t} KL(\hat{p}_{i}^{(m)}||g^{(m)}).
    \label{KL}
\end{equation}
It is noticed that the above objective can be solved directly. The function achieves the minimum when the proposal distribution is equal to the mean of the target distributions:
\begin{equation}
\small
    g^{(m)}=\frac{1}{n_t}\sum_{i=1}^{n_t}\hat{p}_{i}^{(m)},
    \label{eq:score}
\end{equation}

As inferred from Equation~(\ref{eq:score}), in order to reduce the estimated variance, we need to select negative samples that are close to the prototypes of all different categories on average, that is, hard negative samples that are difficult to distinguish by a classifier. More details can be seen in Appendix A.

\subsection{PRD for Contrastive Continual Learning}

Although our method readily provides a more transferable representation space, it is necessary to preserve the learned knowledge of previous tasks for stable representation features and prototype-instance relationships, which is crucial to limit the gap between the proposal distribution and the target distributions. To achieve this objective, we propose a \emph{PRD} to regulate changes in the features of current and previous models through self-distillation \cite{fang2021seed}, which can preserve learned knowledge to overcome Catastrophic Forgetting. Assuming that the samples are drawn from the mini-batch $B$ at the time step $t$ and $Y_{1:t}$ is the class set of $D_{1:t}$, the prototype-instance score vector is the normalized similarity of the sample $x_i$ to the prototypes that appear, which is formally defined as $q(x_j; \theta) = [q_{ij}]_{i\sim Y_{1:t}}$, where $q_{ij}$ denotes the normalized prototype-instance similarity between prototype $i$ and sample $j$:
\begin{equation}
\small
    q_{ij} = \frac{\exp(s_{ij})}{\sum_{i=1}^{|Y_{1:t}|}\exp(s_{ij})}.
    \label{proto-instance sim}
\end{equation}

Our proposed \emph{PRD} loss measures the relation of prototype-instance similarity between the previous and current representation spaces as a self-distillation method. Formally, we denote the parameters of previous and current models as $\theta_{prev}$ and $\theta_{cur}$ respectively, and the \emph{PRD} loss is defined as:
\begin{equation}
\small
    L_{\mathrm{PRD}}(\theta_{cur};\theta_{prev}, B)=\sum_{x_j\sim B}-q(x_j; \theta_{prev})\log q(x_j; \theta_{cur}).
    \label{PRD}
\end{equation}
By using the frozen model with previous representations, \emph{PRD} distill the learned relationship between prototypes and sample features into the current model, helping to keep the target distributions stable to better apply the importance sampling method and preserve previous knowledge better.

\subsection{Objective Function}
Similar to past literature\cite{cha2021co2l}, our loss function is composed of contrastive loss $L_{\mathrm{Sample-NCE}}$ and distillation loss $L_{\mathrm{PRD}}$ with the trade-off parameter $\lambda$ as following:
\begin{equation}
\small
    L(\theta)=L_{\mathrm{Sample-NCE}}(\theta) + \lambda \cdot L_{\mathrm{PRD}}(\theta),
    \label{objective}
\end{equation}
and all the trainable parameters, including prototypes and model parameters, can be updated through optimizing the objective function.

\begin{table*}[!t]
    \centering
    \small
    \begin{tabular}{cccccccc}
	\hline
\multirow{2}{*}
{\textbf{Buffer}}&\textbf{Dataset}&\multicolumn{2}{c}{\textbf{Seq-Cifar-10}}&\multicolumn{2}{c}{\textbf{Seq-Cifar-100}}&\multicolumn{2}{c}{\textbf{Seq-Tiny-ImageNet}}\\
            &\textbf{Scenario}&\textbf{Class-IL}&\textbf{Task-IL}&\textbf{Class-IL}&\textbf{Task-IL}&\textbf{Class-IL}&\textbf{Task-IL}\\
\hline
\multirow{7}{*}{200}&ER&49.16$\pm$2.08&91.92$\pm$1.01&21.78$\pm$0.48&60.19$\pm$1.01&8.65$\pm$0.16&38.83$\pm$1.15\\
    &iCaRL&32.44$\pm$0.93&74.59$\pm$1.24&28.0$\pm$0.91&51.43$\pm$1.47&5.5$\pm$0.52&22.89$\pm$1.83\\
    &GEM&29.99$\pm$3.92&88.67$\pm$1.76&20.75$\pm$0.66&58.84$\pm$1.00&-&-\\
    &GSS&38.62$\pm$3.59&90.0$\pm$1.58&19.42$\pm$0.29&55.38$\pm$1.34&8.57$\pm$0.13&31.77$\pm$1.34\\
    &DER&63.69$\pm$2.35&91.91$\pm$0.51&31.23$\pm$1.38&63.09$\pm$1.09&13.22$\pm$0.92&42.27$\pm$0.90\\
    
    &Co2L&65.57$\pm$1.37&93.43$\pm$0.78&27.73$\pm$0.54&54.33$\pm$0.36&13.88$\pm$0.40&42.37$\pm$0.74\\
    &GCR&64.84$\pm$1.63&90.8$\pm$1.05&33.69$\pm$1.40&64.24$\pm$0.83&13.05$\pm$0.91&42.11$\pm$1.01\\
    &\textbf{CCLIS(Ours)}&\textbf{74.95$\pm$0.61}&\textbf{96.20$\pm$0.26}&\textbf{42.39$\pm$0.37}&\textbf{72.93$\pm$0.46}&\textbf{16.13$\pm$0.19}&\textbf{48.29$\pm$0.78}\\
\hline
\multirow{7}{*}{500}&ER&62.03$\pm$1.70&93.82$\pm$0.41&27.66$\pm$0.61&66.23$\pm$1.52&10.05$\pm$0.28&47.86$\pm$0.87\\
    &iCaRL&34.95$\pm$1.23&75.63$\pm$1.42&33.25$\pm$1.25&58.16$\pm$1.76&11.0$\pm$0.55&35.86$\pm$1.07\\
    &GEM&29.45$\pm$5.64&92.33$\pm$0.80&25.54$\pm$0.65&66.31$\pm$0.86&-&-\\
    &GSS&48.97$\pm$3.25&48.97$\pm$3.25&21.92$\pm$0.34&60.28$\pm$1.18&9.63$\pm$0.14&36.52$\pm$0.91\\
    &DER&72.15$\pm$1.31&93.96$\pm$0.37&41.36$\pm$1.76&71.73$\pm$0.74&19.05$\pm$1.32&53.32$\pm$0.92\\

    &Co2L&74.26$\pm$0.77&95.90$\pm$0.26&36.39$\pm$0.31&61.97$\pm$0.42&20.12$\pm$0.42&53.04$\pm$0.69\\
    &GCR&74.69$\pm$0.80&94.44$\pm$0.32&45.91$\pm$1.30&71.64$\pm$2.10&19.66$\pm$0.68&52.99$\pm$0.89\\
    &\textbf{CCLIS(Ours)}&\textbf{78.57$\pm$0.25}&\textbf{96.18$\pm$0.43}&\textbf{46.08$\pm$0.67}&\textbf{74.51$\pm$0.38}&\textbf{22.88$\pm$0.40}&\textbf{57.04$\pm$0.43}\\
\hline
\end{tabular}

    \caption{Class-IL and Task-IL Continual Learning. We report our performance and the results of rehearsal-based baselines on Seq-Cifar-10, Seq-Cifar-100 and Seq-Tiny-ImageNet with memory sizes 200 and 500, all of which are averaged across ten independent trails.}
\label{tab:acc_result}

\end{table*}

\begin{table*}[!t]
    \centering
    \small
    \begin{tabular}{cccccccc}
	\hline
\multirow{2}{*}
{\textbf{Buffer}}&\textbf{Dataset}&\multicolumn{2}{c}{\textbf{Seq-Cifar-10}}&\multicolumn{2}{c}{\textbf{Seq-Cifar-100}}&\multicolumn{2}{c}{\textbf{Seq-Tiny-ImageNet}}\\
&\textbf{Scenario}&\textbf{Class-IL}&\textbf{Task-IL}&\textbf{Class-IL}&\textbf{Task-IL}&\textbf{Class-IL}&\textbf{Task-IL}\\
\hline
\multirow{4}{*}{200}&DER&35.79$\pm$2.59& 6.08$\pm$0.70&62.72$\pm$2.69&25.98$\pm$1.55 &64.83$\pm$1.48&40.43$\pm$1.05\\
&Co2L&36.35$\pm$1.16& 6.71$\pm$0.35
&67.06$\pm$0.01&37.61$\pm$0.11
 &73.25$\pm$0.21&47.11$\pm$1.04
\\
&GCR&32.75$\pm$2.67& 7.38$\pm$1.02&57.65$\pm$2.48&24.12$\pm$1.17 &65.29$\pm$1.73&40.36$\pm$1.08\\
&\textbf{CCLIS(Ours)}&\textbf{22.59$\pm$0.18}&\textbf{2.08$\pm$0.27
} &\textbf{46.89$\pm$0.59}& \textbf{14.17$\pm$0.20
}&\textbf{62.21$\pm$0.34}&\textbf{33.20$\pm$0.75
}\\
\hline
\multirow{4}{*}{500}&DER&24.02$\pm$1.63&3.72$\pm$0.55 &49.07$\pm$2.54& 25.98$\pm$1.55&59.95$\pm$2.31&28.21$\pm$0.97\\

    &Co2L&25.33$\pm$0.99&3.41$\pm$0.8
 &51.96$\pm$0.80&26.89$\pm$0.45
 &65.15$\pm$0.26&39.22$\pm$0.69
\\
    &GCR&19.27$\pm$1.48&3.14$\pm$0.36 &\textbf{39.20$\pm$2.84}& 15.07$\pm$1.88&56.40$\pm$1.08&27.88$\pm$1.19\\
    &\textbf{CCLIS(Ours)}&\textbf{18.93$\pm$0.61}&\textbf{1.69$\pm$0.12
} &42.53$\pm$0.64& \textbf{12.68$\pm$1.33
}&\textbf{50.15$\pm$0.20}&\textbf{23.46$\pm$0.93
}\\
\hline
\end{tabular}
    \caption{Average Forgetting (lower is better) in Continual Learning, all averaged across five independent trails. For simplicity, we only compare our method with recent baselines.}
\label{tab:forgetting_result}

\end{table*}

\section{Experimental Setup}

\textbf{Research Questions.} 
The remainder of the paper is guided by subsequent research questions: 
(\textbf{RQ1}) Can our proposed method outperform other baselines in various datasets in the continual learning setting? (\textbf{RQ2}) What is the impact of each component, that is, importance sampling and \emph{PRD}, on the performance of our method? (\textbf{RQ3}) What is the connection between the components in our method?

\textbf{Baselines.} We consider two continua learning settings: Class-Incremental Learning (Class-IL) and Task-Incremental Learning (Task-IL). 
%
For evaluation purposes, we take 
the following state-of-the-art baselines:  
(1) ER (Experience Replay)~\cite{riemer2018learning} is a rehearsal-based method with random sampling in memory retrieval and reservoir sampling in memory updates. (2) iCarL \cite{rebuffi2017icarl}, representation learning with an incremental buffer proposes learning representation embeddings and classifiers in the Class-IL setting. (3) GEM \cite{lopez2017gradient} uses episodic memory to minimize negative knowledge transfer and Catastrophic Forgetting. (4) GSS \cite{aljundi2019gradient}, a gradient-based sample selection method that maximizes the variation of the gradients of the replay buffer samples. (5) DER (Dark Experience Replay) and DER++ \cite{buzzega2020dark}, which promotes the consistency of the logit of the current task with the previous one. (6) Co2L (Contrastive Continual Learning)~\cite{cha2021co2l}, a rehearsal-based continual learning algorithm with instance-wise contrastive loss and self-distillation.
(7) GCR (Gradient Coreset-based Replay)~\cite{tiwari2022gcr}, a new gradient-based strategy for \emph{RBS} in continual learning.

\textbf{Datasets.} 
We verify the effectiveness of our method in Class-IL and Task-IL on the following three datasets: Seq-cifar-10 \cite{krizhevsky2009learning}, Seq-cifar-100 \cite{krizhevsky2009learning}, and Seq-tiny-imagenet \cite{le2015tiny}, and all of them are commonly used as benchmarks in previous work\cite{tiwari2022gcr}. Seq-cifar-10 is the set of splits of Cifar10. Following general settings, we divide it into 5 tasks with two classes per task. Seq-cifar-100 is constructed from Cifar100 and divided into five tasks with 20 classes per task, similar to \cite{tiwari2022gcr}. Seq-tiny-imagenet splits Tiny-imagenet \cite{russakovsky2015imagenet} into ten tasks of 20 classes each. Both Class-IL and Task-IL share the same dataset splitting settings.

\textbf{Settings.} 
We train on three datasets using ResNet-18 \cite{he2016deep} as the backbone and report our results with memory sizes 200 and 500. All model parameters, including prototypes and backbone parameters, are trained by backpropagation.
%
To be consistent with previous contrastive continual learning, we freeze model parameters and train a linear classifier to evaluate on samples $R_T$. We adopt the same strategy for training the linear classifier as in~\cite{cha2021co2l} to avoid suffering from the class-imbalanced issue. \emph{Accuracy} and \emph{Average Forgetting} will be adopted to evaluate all the methods in our experiments following the CL literature~\cite{chaudhry2018riemannian, tiwari2022gcr}. Data preparation, model architecture, hyper-parameter selection, and training details, etc., can be referred to Appendix B.

\section{Results and Discussions}
In this section, we conduct sufficient experiments to answer the research questions. Additional comparisons, ablations, and analysis are shown in Appendix C.
\subsection{RQ1:Performance on Squentially Arriving Tasks}
To be consistent with most contrastive continual learning algorithms, we explore whether our method can learn high-quality embeddings to overcome Catastrophic Forgetting. To achieve this, we freeze the backbone parameters and train a linear classifier from scratch to verify the effectiveness of our algorithm. The results in Table \ref{tab:acc_result} show that our method outperforms most state-of-the-art baselines on various datasets in non-stationary task distributions, and the average forgetting results in Table~\ref{tab:forgetting_result} show that our method can effectively alleviate the forgetting issue.

\subsection{RQ2: Ablation Study}
\label{sec:RQ2}

\textbf{Effectiveness of importance sampling.} To validate the effectiveness of the importance sampling method, we introduce three variants of \emph{CCLIS:} Without importance sampling (\emph{IS}) and \emph{PRD}, where we train only with our prototype-based contrastive loss with the random sampling strategy; with importance sampling (\emph{IS}) only, where we optimize the Sample-NCE without self-distill loss; With \emph{PRD} only, where we preserve samples randomly and only use distillation loss to preserve previous knowledge. We compare our methods with the three variants in Seq-Cifar-10 with 200 buffered samples and show the results in Table \ref{tab:acc_ablation} in the Class-IL scenario. It is noticed that without \emph{PRD} loss, importance sampling brings about an improvement of $3.4\%$, while we gain a $1.4\%$ improvement with \emph{PRD} loss on the average.

Furthermore, to illustrate whether \emph{RBS} or the entire importance sampling method leads to improvement, we perform ablation experiments on Seq-Cifar-10 under the Class-IL and Task-IL scenarios. Table \ref{tab:IS_ablation} illustrates that there is a gap between the data distributions of previous tasks and the preserved samples only with \emph{RBS}, which leads to performance degradation; our method via importance sampling can recover the distributions of previous tasks by eliminating bias caused by sample selection, resulting in better performance than random sampling.

\textbf{Effectiveness of \emph{PRD} loss.} Similar to the importance sampling ablation study, we verify the effectiveness of \emph{PRD} loss in Table \ref{tab:acc_ablation}. According to the table,  the improvement brought by \emph{PRD} loss is $31.0\%$ with random sampling, while there is a $28.4\%$ relative improvement related to \emph{PRD} loss with importance sampling. The increased performance growth shows that the \emph{PRD} loss can preserve knowledge from previous tasks to alleviate Catastrophic Forgetting.

\begin{table}[!t]
    \centering
     \small
    \begin{tabular}{cccc}
        \hline  
            &\textbf{Sampling}&\textbf{PRD}&\textbf{200}\\
        \hline
            w/o IS and PRD&Random&$\times$&56.43$\pm$1.54\\
            w/ IS only&IS&$\times$&58.37$\pm$ 1.17\\
            w/ PRD only&Random&$\checkmark$&73.95$\pm$1.12\\
            \textbf{CCLIS(ours)}&IS&$\checkmark$&\textbf{74.95$\pm$0.61}\\
        \hline
\end{tabular}

    \caption{Ablation study of importance sampling and \emph{PRD}. We train our model on the Seq-CIFAR-10 dataset with 200 buffered samples under a Class-IL scenario to explore the effectiveness of importance sampling and \emph{PRD.}}
    \label{tab:acc_ablation}

\end{table}

\begin{table}[!t]
    \centering
    \small
    \begin{tabular}{cccc}
        \hline  
        \multirow{1}{*}&\multirow{1}{*}&\multicolumn{1}{c}{\small{\small{\textbf{200}}}}&\multicolumn{1}{c}{\small{\small{\textbf{500}}}}\\
        \hline
            \multirow{2}{*}{\textbf{Class-IL}}&\small{RBS only} &\small{54.18$\pm0.88$}&\small{67.50$\pm0.50$}\\
            &\small{IS} &\small{\textbf{57.24$\pm$1.46}}&\small{\textbf{68.50$\pm$0.76}}\\
        \hline
            \multirow{2}{*}{\textbf{Task-IL}}&\small{RBS only} &\small{85.00$\pm1.20$}&\small{91.19$\pm0.30$}\\
            &\small{IS} &\small{\textbf{86.28$\pm$0.74}}&\small{\textbf{91.34$\pm$0.20}}\\
        \hline
\end{tabular}
    \caption{Ablation study of the importance sampling method. One experiment trains only with \emph{RBS}, while the other trains with importance sampling. All results are given in Seq-CIFAR-10 under the Class-IL and Task-IL scenarios. The improvements indicate that importance sampling can recover the data distributions of previous tasks, eliminating the bias caused by sample selection.}
    \label{tab:IS_ablation}

\end{table}

\subsection{RQ3: Analysis of the Relationship between Components}

\textbf{Connections between importance sampling and \emph{PRD} loss.} 
As shown in experiments, the results validate that importance sampling and \emph{PRD} loss can preserve the knowledge of previous tasks. We also find that \emph{PRD} brought a large performance boost to the importance sampling method in the ablation study.
The improvement in performance indicates that \emph{ PRD} narrows the gap between the proposal distribution and the target distributions by keeping the relation between prototypes and instances stable, reducing the estimated variance in the sampling of importance.

To further analyze the joint impact of importance sampling and \emph{PRD} loss on performance, we train the model by gradually tuning the trade-off parameter $\lambda$. As in Figure \ref{fig:curve}, the test accuracy gradually increases and stays steady after reaching the highest point with increasing parameters, indicating that \emph{ PRD} complements the importance sampling based contrastive loss by stabilizing the relationship between prototypes and instances in online settings.

\textbf{Connections between importance sampling and hard negative mining.}
By re-examining \emph{RBS}, we find that minimizing the estimated variance of importance sampling is equivalent to preserving hard negatives relative to prototypes of different classes. We have made an intriguing discovery and visualized the embedding space on Seq-Cifar-10 to compare importance sampling with random sampling. The visualization in Figure~\ref{fig:tsne} indicates that our method draws samples distributed at the edges of the clusters and learns better high-quality representations than random sampling.

\begin{figure}[!t]
    \centering
    \includegraphics[scale=0.45]{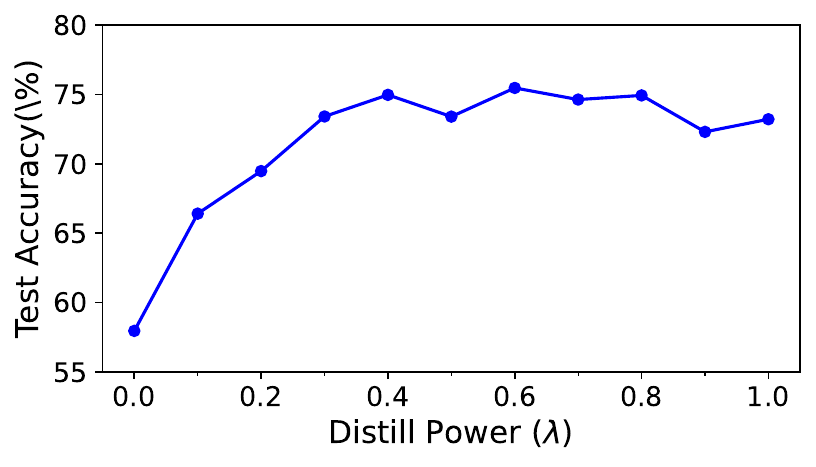}  

    \caption{Performance variant with the distill power $\lambda$ in Seq-CIFAR-10 under Class-IL scenario. \emph{PRD} effectively enhances the performance of importance sampling-based contrastive learning by successfully maintaining the prototype-instance relationship.}
    \label{fig:curve}
\end{figure}

\begin{figure}[!t]
    \centering
    \includegraphics[width=.35\textwidth,height=.30\textwidth]{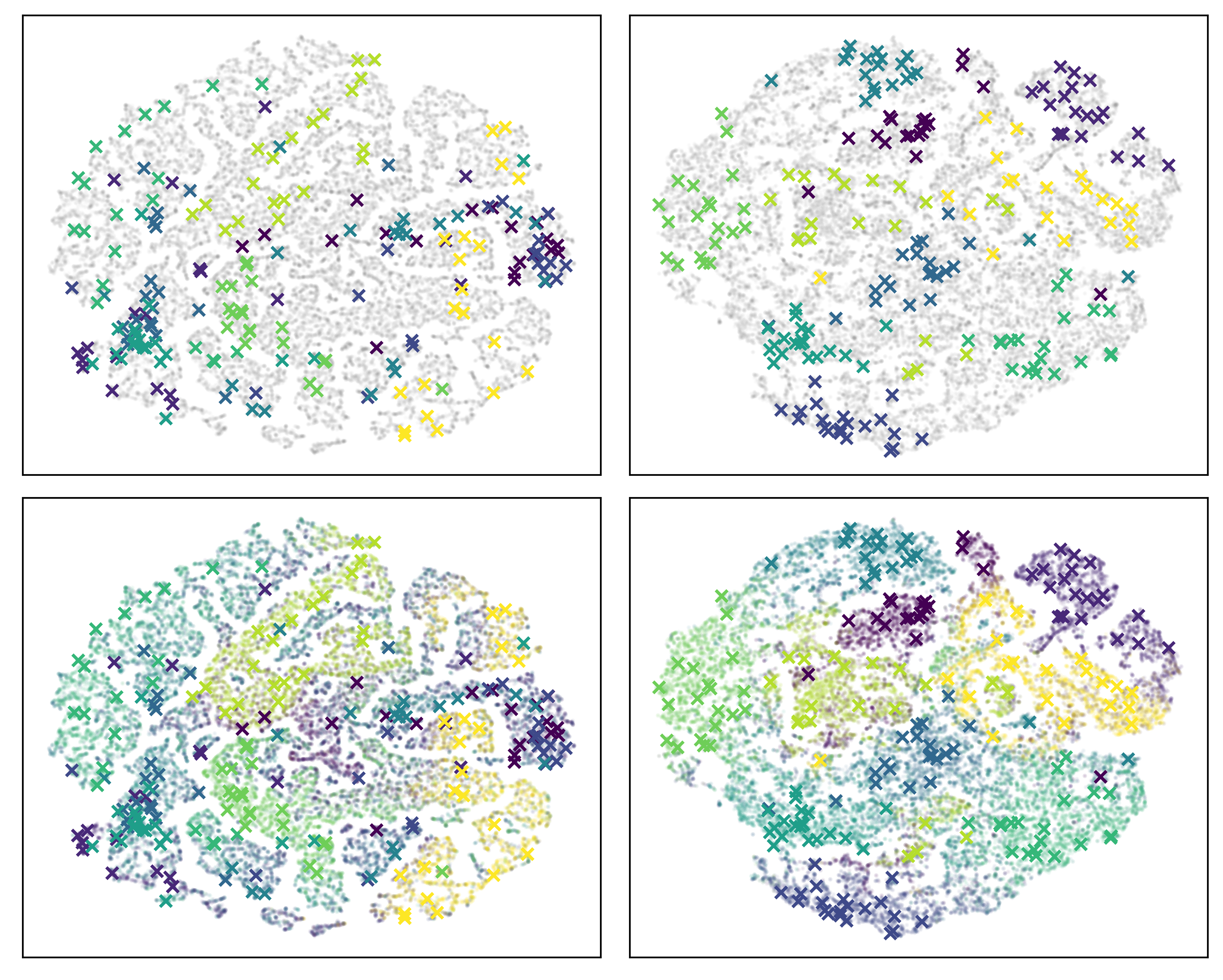}

    \caption{Top: t-SNE visualization of feature embeddings from replay buffer (colored) and all (gray) training samples of Seq-Cifar-10. Bottom: Similar to Top, but all samples are colored to distinguish different clusters clearly. Left: the buffer features drawn by Co2L, a contrastive continual learning algorithm with random sampling, are spread uniformly in clusters. Right: The buffer features sampled by \emph{CCLIS} are mainly distributed at the edge of the clusters. These can be viewed as hard negatives of other classes to help the model learn high-quality contrastive representations.}
    \label{fig:tsne}

\end{figure}

\section{Conclusion, Limitations and Future Work}

In this paper, based on the previous contrastive continual learning method, we introduce an improved contrastive loss through importance sampling, yielding improved contrastive representations online. Furthermore, our \emph{RBS} method, informed by importance weight, retains hard negative samples for future learning. In addition, we propose \emph{PRD} to maintain the prototype-instance relation. Experiments show that our method can better preserve previous task knowledge to overcome Catastrophic Forgetting. 

However, a limitation comes from the recovery of data distributions, since we only recover the distributions of the last visible samples, which is still far from recovering the entire data distribution. We plan to develop algorithms to address this challenge in future work.

\bibliography{aaai24}

\clearpage
\newpage

\section{A. Contrastive Continual Learning with Importance Sampling}
\subsection{Derivation for Sample-NCE}
\label{apx: derivation}
In contrastive continual learning, contrastive loss is utilized to acquire high-quality representations within continual learning scenarios. In this work, we shift our focus to a prototype-based InfoNCE, denoted as ProtoNCE, to improve representation embeddings:
\begin{equation}
    \sum_{i=1}^N-\log\frac{\exp(s_{y_ii})}{\sum_{j=1}^N\exp(s_{y_ij})}.
    \label{apx: origin protoNCE}
\end{equation}

Given the constraints of our computational resources, it is not feasible to train our model on the entire dataset in an online setting. To address the Catastrophic Forgetting challenge, we employ a constrained cache to retain a subset of samples, akin to prior rehearsal-based methods. Additionally, we adapt the ProtoNCE approach to accommodate training on tasks that arrive sequentially:
\begin{equation}
   \sum_{i\sim \hat{C}_{t-1} \cup Y_t}\sum_{j\sim S_i}-\log\frac{\exp(s_{ij})}{\sum_{k\sim R_{t-1}\cup D_t}\exp(s_{ik})}.
   \label{apx: CacheNCE}
\end{equation}

To simplify the deduction, we focus on the specific prototype $c_i$ of the previous task and the sample $j$ drawn from $S_i$, then the loss gradient~\ref{apx: CacheNCE} has the following form:
\begin{equation}
\begin{split}
    \nabla_\theta L_{i,j}=&-\nabla_\theta s_{ij}+\sum_{k\sim S_i\cup D_t} p_{ik}\nabla_\theta s_{ik}\\
    &+\sum_{m\sim \hat{C}_{t-1}\backslash i}\rho_{i}^{(m)}\sum_{k\sim S_{m}} \tilde{p}_{ik}\nabla_\theta s_{ik},
    \label{gradient}
\end{split}
\end{equation}
where $p_{ij}=\frac{\exp(s_{ij})}{\sum_{k\sim R_{t-1}\cup D_t}\exp(s_{ik})}$ and $\tilde{p}_{ij}=\frac{p_{ij}}{\sum_{k\sim S_{y_j}}p_{ik}}$ with the normalized function $\rho_{i}^{(m)}:=\sum_{k\sim S_m}p_{ik}$.

Since the samples of current tasks are not available during the replay buffer selection, we can express the sum of the negative gradients associated with past tasks, denoted by the last term of Equation~\ref{gradient}, as the expectation over these gradients:
\begin{equation}
\begin{split}
    \nabla_\theta L_{i,j}=&-\nabla_\theta s_{ij}+\sum_{k\sim S_i\cup D_t} p_{ik}\nabla_\theta s_{ik}\\
    &+\sum_{m\sim \hat{C}_{t-1}\backslash i}\rho_{i}^{(m)}\mathbb{E}_{k\sim\tilde{p}_{i}^{(m)}}\nabla_\theta s_{ik}.
    \label{gradient_E}
\end{split}
\end{equation}
Then we apply importance sampling to the expectation of the negative gradients in Equation~\ref{gradient_E} and obtain the following:
\begin{align}
\small
\nabla_\theta L_{i,j}&\approx -\nabla_\theta s_{ij}+\sum_{k\sim S_i\cup D_t} p_{ik}\nabla_\theta s_{ik} \displaybreak[3] \nonumber\\
    &\quad\;+\sum_{m\sim \hat{C}_{t-1}\backslash i}\rho_{i}^{(m)}\mathbb{E}_{k\sim g^{(m)}}(\frac{\tilde{p}_{ik}^{(m)}}{g_k^{(m)}}\nabla_\theta s_{ik}) \nonumber\\
    &= -\nabla_\theta s_{ij}+\sum_{k\sim S_i\cup D_t} \frac{\exp{s_{ik}}}{\sum_{l\sim R_{t-1}\cup D_t}\exp{s_{il}}}\nabla_\theta s_{ik} \displaybreak[3] \nonumber\\
    &\quad\;+\sum_{m\sim \hat{C}_{t-1}\backslash i}\mathbb{E}_{k\sim g^{(m)}}(\frac{\exp{s_{ik}}/g_k^{(m)}}{\sum_{l\sim R_{t-1}\cup D_t}\exp{s_{il}}}\nabla_\theta s_{ik}).\nonumber\\
    \label{gradient_IS}
\end{align}
However, since the new tasks are not yet available during the sample selection, we cannot directly apply unbiased importance sampling to estimate the gradient~\ref{gradient_E}. To address this challenge, we estimate the denominator $W:=\sum_{j\sim R_{t-1}\cup D_t}\exp{s_{ik}}$ of Equation~\ref{gradient_IS} with importance sampling, which can be viewed as an average with uniform distribution $U(|S_m|)$ for specific class $m$:
\begin{align}
\small
    W&=\sum_{k\sim S_i\cup D_t}\exp{s_{ik}}+\sum_{m\sim \hat{C}_{t-1}\backslash i}\sum_{k\sim S_{m}}\exp{s_{ik}} \nonumber\displaybreak[3]\\
    &\approx\sum_{k\sim S_i\cup D_t}\exp{s_{ik}}+\sum_{m\sim \hat{C}_{t-1}\backslash i}|S_m|\mathbb{E}_{k\sim U(|S_m|)}\exp{s_{ik}}\nonumber\displaybreak[3]\\
    &=\sum_{k\sim S_i\cup D_t}\exp{s_{ik}}+\sum_{m\sim \hat{C}_{t-1}\backslash i}\mathbb{E}_{k\sim g^{(m)}}\frac{\exp{s_{ik}}}{g_k^{(m)}}. \label{apx: eq_total_E}
\end{align}

Assuming that for each $m$, we draw and preserve $|J_m|$ samples and apply \emph{Monte-Carlo} to estimate the expectation with buffered samples. The approximation of $W$ is made as the following formula:
\begin{equation}
    \hat{W}\approx \sum_{k\sim J_i\cup D_t}\exp{s_{ik}}+\sum_{m\sim \hat{C}_{t-1}\backslash i}\frac{1}{|J_m|}\sum_{k\sim J_m}(\exp{s_{ik}})/g_k^{(m)}.
    \label{apx: eq_W_E}
\end{equation}

Using this estimator, we can apply the biased importance sampling method to the expectation of weighted gradients in the Equation~\ref{apx: eq_total_E} and make the approximation of the average gradients. Here we assume that the drawn samples $J_i$ could represent the samples from class $i$ and we use $\mu_{ij}:=-\nabla_\theta s_{ij}+\sum_{k\sim R_{t-1}\cup D_t\backslash(S_i\backslash J_i)} p_{ik}\nabla_\theta s_{ik}$ to approximate the gradient:
\begin{align}
\small
& \nabla_\theta L_{i,j} \approx \mu_{ij} \approx -\nabla_\theta s_{ij}+\sum_{k\sim J_i\cup D_t} \frac{\exp{s_{ik}}}{\hat{W}}\nabla_\theta s_{ik} \displaybreak[3] \nonumber\\
    &+\sum_{m\sim \hat{C}_{t-1}\backslash i} \frac{1}{|J_m|}\sum_{k\sim J_m}(\frac{\exp{s_{ik}}/g_k^{(m)}}{\hat{W}}\nabla_\theta s_{ik}).
    \label{apx: gradient_IS}
\end{align}

Finally, the Sample-NCE loss can be obtained as the anti-derivative of the gradient in the specific task $t$:
\begin{equation}
    L_{\mathrm{Sample-NCE}}(\theta;t) =\sum_{i\sim Y_t}\sum_{j\sim S_i} L_{i,j} + \sum_{i\sim \hat{C}_t} \sum_{j\sim J_i} \hat{L}_{i,j},
    \label{eq:total}
\end{equation}
where
\begin{small}
    \begin{equation}
    \hat{L}_{i,j}=-\log\frac{\exp(s_{ij})}{\sum\limits_{k\sim J_i\cup D_t}\exp(s_{ik})+\sum\limits_{m\sim \hat{C}_{t-1}\backslash i}\sum\limits_{k\sim S_{m}}\frac{\exp(s_{ik})}{g_{k}^{(m)}|J_m|}}.
    \label{contrastive_IS}
    \end{equation}
\end{small}

In particular, assuming that for class m, the sampling probability of the $|J_m|$ sampled instances is uniformly distributed as $1/|J_m|$, we can observe that our loss function degenerates into the loss function of prototype-based contrastive continual learning with random sampling. In other words, the original contrastive continual learning can be regarded as a special case of our method.

\subsection{Variance Estimation for Biased Importance Sampling}
Regarding the biased importance sampling method, the estimation error between the estimator and ground truth is primarily dependent on the bias and estimation variance, which are usually intractable. To propose a criterion for the selection of the proposal distribution, we use the delta method~\cite{liu2008monte} to provide an approximation for the estimation error. In this subsection, we mainly prove the following theorem to illustrate that we can minimize the estimation error by reducing the estimation variance of the importance weight.

\begin{theorem}
Assuming that the gradients of score functions are bounded, i.e., $||\nabla_{\theta}s_{ij}||_2 \leq M, \forall i,j$, we can have the following bound on the mean square error between the estimator $\hat{\mu}_{ij}$ and the gradient $\mu_{ij}$ for specific prototype $i$ and sample $j$:
\begin{align}
    \mathrm{MSE}(\hat{\mu}_{ij})\leq \sum_{m\sim C_{t-1}\backslash i}\frac{M^2}{|J_m|}(1+\mathrm{var}_{g^{(m)}}\frac{p_i^{(m)}}{g^{(m)}}),
\end{align}
where $g^{(m)}$ is the proposal distribution and $\omega_i^{(m)}$ is the importance weight for the specific class $m$.
\end{theorem}

\begin{proof}
With the notations in the former section, we mainly focus on the gradient estimator proposed in equation~(\ref{apx: gradient_IS}) for specific prototype $i$:
\begin{align}
\small
\hat{\mu}_{ij}= &-\nabla_\theta s_{ij}+\sum_{k\sim J_i\cup D_t} \frac{\exp{s_{ik}}}{\hat{W}}\nabla_\theta s_{ik} \displaybreak[3] \nonumber\\
    &+\sum_{m\sim \hat{C}_{t-1}\backslash i} \frac{1}{|J_m|}\sum_{k\sim J_m}(\frac{\exp{s_{ik}}/g_k^{(m)}}{\hat{W}}\nabla_\theta s_{ik}).
    \label{apx: grad_estimator}
\end{align}
We denote $h_{ik}:= \nabla_\theta s_{ik}$, $\omega_{ik}^{(m)}:=\frac{\tilde{p}_{ik}}{g_{k}^{(m)}}$ and $z_{ik}^{(m)}:=\omega_{ik}^{(m)}h_{ik}$, where $\mathbb{E}_{g^{(m)}}\omega_i^{(m)} = 1$. Notice that $\Tilde{\mu}_i:=\mu_{ij} - (-\nabla_\theta s_{ij})$ is independent with $j$, we can ignore the constant and focus on the rest of the estimation:
\begin{align}
\small
&\frac{\sum_{m\sim \hat{C}_{t-1}\backslash i}\frac{\rho_i^{m}}{|J_m|}\sum_{k\sim J_m}\omega_{ik}^{(m)}h_{ik}+\sum_{k\sim J_i\cup D_t}p_{ik}h_{ik}}{\sum_{m\sim \hat{C}_{t-1}\backslash i}\frac{\rho_i^{m}}{|J_m|}\sum_{k\sim J_m}\omega_{ik}^{(m)}+\sum_{k\sim J_i\cup D_t}p_{ik}}\nonumber\\
&=\frac{\sum_{m\sim \hat{C}_{t-1}\backslash i}\rho_i^{m}(\Bar{z}_i^{(m)}-\mathbb{E}_{g}\Bar{z}_i^{(m)})+\Tilde{\mu}_i}{\sum_{m\sim \hat{C}_{t-1}\backslash i}\rho_i^{m}(\Bar{\omega}_i^{(m)}-1)+1}\equiv {\hat{\hat{\mu}}_i},
\label{apx: neg_part}
\end{align}
where $\Bar{z}_i^{(m)}$ and $\Bar{\omega}_i^{(m)}$ are corresponding averages and $\hat{\hat{\mu}}_i$ is an estimate of $\Tilde{\mu}_i$. By the delta method, we can calculate the expectation of the gradient estimator:
\begin{align}
\mathbb{E}_{g}\hat{\mu}_{ij}\approx &\mu_{ij}- \sum_{m\sim \hat{C}_{t-1}\backslash i}\frac{1}{|J_m|}(\rho_i^{(m)})^2 cov_{g^{(m)}}(\omega_{i}^{(m)},z_{i}^{(m)})\nonumber\displaybreak[3]\\
    &+\sum_{m\sim \hat{C}_{t-1}\backslash i}\frac{1}{|J_m|}(\rho_i^{(m)})^2\Tilde{\mu}_{i}var_{g^{(m)}}\omega_{i}^{(m)}, 
\end{align}
Similarly, we can use the standard delta method to estimate the variance of the gradient estimator:
\begin{align}
    var_{g}\hat{\mu}_{ij} \approx& \sum_{m\sim \hat{C}_{t-1}\backslash i}\frac{(\rho_i^{(m)})^2}{|J_m|}\bigg[{\Tilde{\mu}_{i}}^2var_{g^{(m)}}(\omega_{i}^{(m)})\nonumber\\
    &+var_{g^{(m)}}(z_{i}^{(m)})-2{\Tilde{\mu}_{i}}cov_{g^{(m)}}(\omega_{i}^{(m)},z_{i}^{(m)})\bigg].
\end{align}
With the above approximation of the expectation and variance of gradient, we use the bias-variance decomposition to indicate that the bias between the estimator and the ground truth is not necessarily small and the mean square error mainly depends on the value of the estimation variance:
\begin{align}
\small
    & MSE(\hat{\mu}_{ij}) = [\mathbb{E}_{g}(\hat{\mu}_{ij})-\mu_{ij}]^2\nonumber\\ 
    &\hspace{0.8in}+ \sum_{m\sim \hat{C}_{t-1}\backslash i}var_{g^{(m)}}(\hat{\mu}_{ij}^{(m)})\nonumber \\
    &\hspace{0.65in}\approx \sum_{m\sim \hat{C}_{t-1}\backslash i}var_{g^{(m)}}(\hat{\mu}_{ij}^{(m)}) + O(|J_m|^{-2}).
\end{align}
To further analyze the mean square error, we need to reformulate $MSE(\hat{\mu}_{ij})$ as a function of $var(\omega)$. Denoting $H^{(m)}=h^{(m)}$, we can have $z^{(m)}=\omega^{(m)}H^{(m)}$, $\mu=\mathbb{E}_{g^{(m)}}(\omega^{(m)}H^{(m)})$ and finally estimate the variance of the estimator $\hat{\mu}_i$ using the delta method as follows:
\begin{align}
    var_{g^{(m)}}(\hat{\mu}_{ij})& \approx  \sum_{m\sim \hat{C}_{t-1}\backslash i} \bigg\{\frac{(\rho_i^{(m)})^2}{|J_m|}var_{\pi}(H_i^{(m)}) \nonumber\displaybreak[3] \\
    & \quad\times \{1+var_{g^{(m)}}(\omega_i^{(m)})\} \bigg\}.\\
    &=  \sum_{m\sim \hat{C}_{t-1}\backslash i} \bigg\{\frac{1}{|J_m|}var_{\pi}(H_i^{(m)}) \nonumber\displaybreak[3] \\
    & \quad\times \{(\rho_i^{(m)})^2+var_{g^{(m)}}(\frac{p_i^{(m)}}{g^{(m)}})\} \bigg\}.
\end{align}
By noting that $\rho_i^{(m)}\leq 1$ the variation of $H$ can be bounded as:
\begin{equation}
 var_\pi(H)=\mathbf{E}_{\pi}H^2-(\mathbf{E}_{\pi}H)^2 \leq \max|H^2| = M^2,
\end{equation}
we can finally obtain a bound for the estimated error:
\begin{equation}
    \mathrm{MSE}(\hat{\mu}_{ij})\leq \sum_{m\sim C_{t-1}\backslash i}\frac{M^2}{|J_m|}\big(1+\mathrm{var}_{g^{(m)}}\frac{p_i^{(m)}}{g^{(m)}}\big).
\end{equation}
\end{proof}

\begin{table*}[!t]
    \centering  
    \begin{tabular}{ccc}
	\hline
\textbf{Method}&\textbf{Buffer}&\textbf{hyper-parameters}\\
\hline
\multicolumn{3}{c}{\textbf{Seq-Cifar-10}}\\
\hline
\multirow{2}{*}{CCLIS}&200&\multirow{2}{*}{$\eta:1.0,\eta_{proto}:0.01, bsz:512, \tau:0.5,\kappa_{past}:0.1,\kappa_{cur}:0.2,E_{t=0}:500, E_{t\geq 0}:100$}\\
&500&\\
\hline
\multicolumn{3}{c}{\textbf{Seq-Cifar-100}}\\
\hline
\multirow{2}{*}{Co2L}&200&\multirow{2}{*}{$\eta:0.5, bsz:512, \tau:0.5,\kappa_{past}:0.01,\kappa_{cur}:0.2,E_{t=0}:500, E_{t\geq 0}:100$}\\
&500&\\
\multirow{2}{*}{CCLIS}&200&\multirow{2}{*}{$\eta:1.0,\eta_{proto}:0.01, bsz:512, \tau:0.5,\kappa_{past}:0.1,\kappa_{cur}:0.2,E_{t=0}:500, E_{t\geq 0}:100$}\\
&500&\\
\hline
\multicolumn{3}{c}{\textbf{Seq-Tiny-Imagenet}}\\
\hline
\multirow{2}{*}{CCLIS}&200&\multirow{2}{*}{$\eta:1.0,\eta_{proto}:0.01, bsz:512, \tau:0.5,\kappa_{past}:0.1,\kappa_{cur}:0.2,E_{t=0}:500, E_{t\geq 0}:50$}\\
&500&\\
\hline
\end{tabular}
    \caption{hyper-parameters selected in our experiments.}
\label{tab:hyper_selection}
\end{table*}

\subsection{Criterion for Proposal Distributions Selection}
Here we aim at providing a proposal distribution selection method to minimize the mean of the estimated error over all the target distributions, which can be bounded as the function of the mean of the variance:
\begin{align}
    & \mathbb{E}_i  MSE(\hat{\mu}_i) \leq \sum_{i=1}^{n_t}\sum_{m\sim C_{t-1}\backslash i}\frac{M^2B^2}{|J_m|n_t}(1+\mathrm{var}_{g^{(m)}}\frac{p_i^{(m)}}{g^{(m)}})\nonumber\displaybreak[3]\\
    &= \sum_{m}\frac{M^2B^2}{|J_m|}\mathbb{E}_i\mathrm{var}_{g^{(m)}}\frac{p_i^{(m)}}{g^{(m)}}+const,
\end{align}
where $\hat{\mu}_i:=\sum_j\hat{\mu}_{ij}$ is the gradient of~\ref{apx: CacheNCE} for specific $i$ and $B$ is the batch size for training. 
For a specific class $m$, we can see that the bound is proportional to the expectation of the variance of the importance weight. However, for stability, we propose to minimize the \emph{KL-divergence} which can measure the gap between proposal distributions and target distributions, rather than minimizing the average variance~\cite{ortiz2013adaptive, su2021variational}. The normalizing function of $p_i$ is noticed to be intractable, as the samples of the new task $D_t$ are invisible when we draw samples from $R_{t-1}$. For this reason, we use $\hat{p}_{ij} = \frac{\exp(s_{ij})}{\sum_{k\sim R_{t-1}}\exp(s_{ik})}$ to approximate $p_i$:
\begin{equation}
    \min\frac{1}{n_t}\sum_{i=1}^{n_t} KL(\hat{p}_{i}^{(m)}||g^{(m)}),\quad s.t.\sum_{k}g_k^{(m)}=1,
    \label{KL}
\end{equation}
where $n_t$ is the number of data distributions of the visible classes of $R_{t-1}$ except for the class $m$.
Noticing that both proposal distributions and target distributions are discrete, we can thus apply the \emph{Lagrange Multiplier Method} to find an appropriate sampling distribution to minimize the \emph{KL-divergence:}
\begin{equation}
    g^{(m)} = \arg\min_g \frac{1}{n_t}\sum_{i=1}^{n_t} KL(\hat{p}_{i}^{(m)}||g)+\alpha(\sum_{k}g_k-1),
\end{equation}
where $\alpha$ is a Lagrange multiplier. We can finally obtain the sampling distribution by solving the Lagrange function:
\begin{equation}
    g^{(m)}=\frac{1}{n_t}\sum_{i=1}^{n_t}\hat{p}_{i}^{(m)}.
    \label{eq:score}
\end{equation}

\section{B. Experimental Details}
We implement our model based on the code of Co2L~\cite{cha2021co2l}. 
We follow most of the continual learning settings, model architecture, data preparation, and other training details from the past literature, especially Co2L and GCR~\cite{tiwari2022gcr} when conducting our experiments.

\subsection{Model Architecture}
Similar to previous work, samples are sequentially forwarded to the encoder $f_\theta$ and the projection layer $g_\phi$ to obtain the responding embeddings, and then enter the prototype layer to compute the contrastive loss with visible prototypes. We follow the backbone settings claimed in Co2L and GCR on all the datasets that Resnet-18 is used as the backbone for representation learning, and a two-layer projection MLP is used to map the output of the backbone into a 128-dimension embedding space, where embeddings are normalized as unit vectors through the normalization layer.

During training, a trainable prototype layer is constructed with a linear layer. The normalized embeddings are forwarded to calculate the score as the dot product of the corresponding prototype and instance embeddings. All the model parameters are initialized at random. During testing, the prototype layer is removed and the rest of the encoder is frozen to retain the representation knowledge learned in the training period. We use a linear layer as the classifier to measure the quality of contrastive representations in image classification.

\subsection{Data Preparation}
In our study under the contrastive continual framework, we have made significant advances over previous work. First, our prototype-based InfoNCE loss eliminates the need to augment each sample into two views, thus reducing memory usage. Second, to prevent the occurrence of batches containing only a single class, which can adversely affect our contrastive algorithm, we draw each mini-batch in equal proportions from both the replay buffer and the current task. Lastly, enhancing our model's robustness requires various augmentation techniques. For this, we have adopted a set of augmentations that are standard in the contrastive continual learning literature.

\begin{table*}[!t]
    \centering
    \small
    \begin{tabular}{cccccccc}
	\hline
\multirow{2}{*}
{\textbf{Buffer}}&\textbf{Dataset}&\multicolumn{2}{c}{\textbf{Seq-Cifar-10}}&\multicolumn{2}{c}{\textbf{Seq-Cifar-100}}&\multicolumn{2}{c}{\textbf{Seq-Tiny-ImageNet}}\\
&\textbf{Scenario}&\textbf{Class-IL}&\textbf{Task-IL}&\textbf{Class-IL}&\textbf{Task-IL}&\textbf{Class-IL}&\textbf{Task-IL}\\
\hline
\multirow{8}{*}{200}
&ER&59.3$\pm$2.48&6.07$\pm$1.09&75.06$\pm$0.63&	27.38$\pm$1.46	&76.53$\pm$0.51&	40.47$\pm$1.54
 \\
&GEM&80.36$\pm$5.25&9.57$\pm$2.05&77.4$\pm$1.09&	29.59$\pm$1.66
&-&- \\ 
&GSS&72.48$\pm$4.45&	8.49$\pm$2.05&	77.62$\pm$0.76&	32.81$\pm$1.75&	76.47$\pm$0.4&	50.75$\pm$1.63
 \\
&iCARL&23.52$\pm$1.27&	25.34$\pm$1.64&	47.2$\pm$1.23&	36.2$\pm$1.85&	\textbf{31.06$\pm$1.91}&	42.47$\pm$2.47
 \\
&DER&35.79$\pm$2.59& 6.08$\pm$0.70&62.72$\pm$2.69&25.98$\pm$1.55 &64.83$\pm$1.48&40.43$\pm$1.05\\
&Co2L&36.35$\pm$1.16& 6.71$\pm$0.35
&67.06$\pm$0.01&37.61$\pm$0.11
 &73.25$\pm$0.21&47.11$\pm$1.04
\\
&GCR&32.75$\pm$2.67& 7.38$\pm$1.02&57.65$\pm$2.48&24.12$\pm$1.17 &65.29$\pm$1.73&40.36$\pm$1.08\\
&\textbf{CCLIS(Ours)}&\textbf{22.59$\pm$0.18}&\textbf{2.08$\pm$0.27
} &\textbf{46.89$\pm$0.59}& \textbf{14.17$\pm$0.20
}&62.21$\pm$0.34&\textbf{33.20$\pm$0.75
}\\
\hline
\multirow{8}{*}{500}
&ER&43.22$\pm$2.1&3.5$\pm$0.53&67.96$\pm$0.78&	17.37$\pm$1.06	&75.21$\pm$0.54&	30.73$\pm$0.62
 \\
&GEM&78.93$\pm$6.53&5.6$\pm$0.96&71.34$\pm$0.78&	20.44$\pm$1.13&-&-
 \\
&GSS&59.18$\pm$4.0&	6.37$\pm$1.55&	74.12$\pm$0.42&	26.57$\pm$1.34&	75.3$\pm$0.26	&45.59$\pm$0.99
 \\
&iCARL&28.2$\pm$2.41&	22.61$\pm$3.97&	40.99$\pm$1.02&	27.9$\pm$1.37	&\textbf{37.3$\pm$1.42}	&39.44$\pm$0.84
 \\
&DER&24.02$\pm$1.63&3.72$\pm$0.55 &49.07$\pm$2.54& 25.98$\pm$1.55&59.95$\pm$2.31&28.21$\pm$0.97\\

    &Co2L&25.33$\pm$0.99&3.41$\pm$0.8
 &51.96$\pm$0.80&26.89$\pm$0.45
 &65.15$\pm$0.26&39.22$\pm$0.69
\\
    &GCR&19.27$\pm$1.48&3.14$\pm$0.36 &\textbf{39.20$\pm$2.84}& 15.07$\pm$1.88&56.40$\pm$1.08&27.88$\pm$1.19\\
    &\textbf{CCLIS(Ours)}&\textbf{18.93$\pm$0.61}&\textbf{1.69$\pm$0.12
} &42.53$\pm$0.64& \textbf{12.68$\pm$1.33
}&50.15$\pm$0.20&\textbf{23.46$\pm$0.93
}\\
\hline
\end{tabular}
    \caption{Average Forgetting compared with all baselines in Continual Learning. }
\label{tab:add_forgetting_result}

\end{table*}

\begin{table*}[htbp]
    \centering
    \resizebox*{!}{0.2\columnwidth}{
    \begin{tabular}{ccccccc}
        \hline  
&\textbf{50}&\textbf{100}&\textbf{200}&\textbf{300}&\textbf{400}&\textbf{500}\\
        \hline
            w/o IS and PRD&41.18$\pm$0.11&48.33$\pm$0.13&56.43$\pm$1.54&59.97$\pm$1.65&66.26$\pm$0.50&68.11$\pm$0.27\\
            w/ IS only&35.43$\pm$1.11&49.18$\pm$1.17&58.37$\pm$1.17&61.14$\pm$1.26&65.59$\pm$0.60&68.88$\pm$0.66\\
            w/ PRD only&\textbf{68.33$\pm$0.04}&70.75$\pm$0.76&73.95$\pm$1.12&74.62$\pm$0.37&77.35$\pm$0.11&78.13$\pm$0.31\\
\textbf{CCLIS(ours)}&67.09$\pm$0.69&\textbf{72.57$\pm$1.60}&\textbf{74.95$\pm$0.61}&\textbf{76.02$\pm$0.12}&\textbf{77.40$\pm$0.28}&\textbf{78.57$\pm$0.25}\\
        \hline
    \end{tabular}}
    \caption{Different buffer sizes on Seq-Cifar-10 with Resnet18. }
    
    \label{tab:buffer_abla}
\end{table*}

\begin{table}[htbp]
    \centering
    \resizebox*{!}{0.3\columnwidth}{
    \begin{tabular}{cccc}
        \hline  
            &&\textbf{100}&\textbf{300}\\
        \hline
            \multirow{2}{*}{Resnet18}&w/o IS+PRD&48.33$\pm$0.13&59.97$\pm$1.65\\
            &ours&\textbf{72.57$\pm$1.60}&\textbf{76.02$\pm$0.12}\\
            \multirow{2}{*}{Resnet34}&w/o IS+PRD&46.43$\pm$2.05&58.65$\pm$1.43\\
            &ours&\textbf{71.88$\pm$0.96}&\textbf{75.83$\pm$0.37}\\
            \multirow{2}{*}{Resnet50}&w/o IS+PRD&45.83$\pm$0.28&57.38$\pm$0.13\\
            &ours&\textbf{73.23$\pm$1.36}&\textbf{75.92$\pm$0.11}\\
        \hline
    \end{tabular}
    }
    \vspace{-0.5em}
    \caption{Different architectures on Seq-Cifar-10.}
    \label{tab:resnet_abla}
\end{table}

\subsection{Hyper-parameter Selection} 
The selection of hyper-parameters is a crucial step in training robust models. The selection criterion was determined for our experiments using test accuracy obtained from a model trained on the validation set. This validation set comprises $10\%$ of the samples drawn from the primary training set. The model was trained for each task over 50 epochs, and the test accuracy was averaged over five independent trials.
The hyper-parameters under consideration were: 

\begin{itemize} 
\item Learning rate ($\eta$) selected from ${0.1, 0.5, 1.0}$ 
\item Prototype learning rate ($\eta_{\text{proto}}$) from ${0.01, 0.1, 0.5}$ 
\item Batch size ($bsz$) from ${256, 512}$ \item Temperature for the Sample-NCE loss ($\tau$) from ${0.1, 0.5, 1.0}$
\item Past temperature ($\kappa_{\text{past}}$) from ${0.1, 0.5, 1.0}$ 
\item Current temperature ($\kappa_{\text{cur}}$) from ${0.1, 0.2}$ used for prototype-instance relation distillation loss 
\item Number of start epochs ($E_{t=0}$) fixed at ${500}$ 
\item Number of epochs ($E_{t\geq 0}$) chosen from ${50, 100}$ 
\end{itemize}

All hyper-parameters selected for our method and baseline comparisons are detailed in Table \ref{tab:hyper_selection}. It is important to note that certain hyper-parameters previously explored in the literature have been omitted for brevity. For our experiments within \emph{CCLIS}, a distill power ($\lambda$) of 0.6 was used, while Co2L used a distill power of 1.0.

\begin{table}[htbp]
    \centering
    \scalebox{0.85}{
    \begin{tabular}{ccccc}
        \hline  
            \multirow{2}{*}&\multicolumn{2}{c}{\small{\textbf{Seq-CIFAR-10}}}&\multicolumn{2}{c}{\small{\textbf{Seq-CIFAR-100}}}\\
            &\small{\textbf{200}}&\small{\textbf{500}}&\small{\textbf{200}}&\small{\textbf{500}}\\
        \hline
            \small{$L_{asym}^{sup}$(Co2L)} &\small{53.57$\pm$1.03}&\small{59.52$\pm$0.69}&\small{25.99$\pm$0.45}&\small{33.61$\pm$0.52}\\
            \small{$L^{sup}$(Ours)}&\small{\textbf{56.43$\pm$1.54}}&\small{\textbf{68.23$\pm$0.28}}&\small{\textbf{27.89$\pm$0.47}}&\small{\textbf{36.91$\pm$0.44}}\\
        \hline
\end{tabular}}
\caption{Performance of our symmetric prototype-based NCE loss($L^{sup}$) \emph{vs} the asymmetric supervised contrastive loss($L_{asym}^{sup}$) proposed by Co2L, both of which are trained on Seq-CIFAR-10 and Seq-CIFAR-100 without self-distill and importance sampling methods under the Class-IL scenario. }
    \label{tab:NCE_ablation}
\end{table}

\subsection{Training Details} 
We employ a linear warm-up for the initial 10 epochs during the representation learning phase. Subsequently, the learning rate decays following the cosine decay schedule as proposed by~\citet{loshchilov2016sgdr}. Training is performed using the Stochastic Gradient Descent (SGD) optimizer, complemented with a momentum of 0.9 and a weight decay of 0.0001 in all trials.

We randomly choose a class from the set of classes in the final period and then uniformly select a sample from the specific class to train the classifier. We train the linear classifier over 100 epochs in the linear evaluation stage. This is done using SGD, which is characterized by a momentum of 0.9. In particular, we omit weight decay during this phase. The learning rate undergoes an exponential decay, with decay rates set at 0.2 for the 60th, 75th, and 90th epochs. We use \{0.5, 0.1, 0.5\} learning rate for the linear classifier in datasets \{Seq-Cifar-10, Seq-Cifar-100, Seq-Tiny-Imagenet\}. It is pertinent to mention that these training specifics are aligned with those of Co2L, ensuring a fair basis for comparison.

Furthermore, to mitigate any influence on the importance sampling technique arising from data augmentation, we compute the score employing static model parameters across all observable samples. This measure ensures that the sample score remains consistent throughout the training process. We perform experiments five times as a result of the diverse data augmentations applied to samples. By averaging the sample scores across these runs, we derive the final importance weight, bolstering the robustness of our algorithms. Subsequently, we sample uniformly from the proposal distribution without repetition and store the resultant data in the replay buffer.

\section{C. Additional Experiments}
\textbf{Average Forgetting.} Here we represent additional results about the average forgetting. We compare our method with all seven baselines and Table~\ref{tab:add_forgetting_result} shows that our method can mitigate Catastrophic Forgetting better than most of the baselines.

\textbf{Model performance with different buffer sizes.} We conduct additional experiments to explore the effectiveness of our algorithm with different buffer sizes. As shown in~\ref{tab:buffer_abla}, 
when the number of cached samples is too low, our method exhibits significant estimation bias, resulting in performance degradation. As the number of cached samples gradually increases, our method demonstrates stable superiority over methods based on random sampling within a certain range. However, as the number of cached samples continues to increase, the advantage of our method gradually diminishes. This suggests that our method is suitable for a certain range of cached sample quantities.

\textbf{Model performance with different architectures.} To investigate the impact of different model architectures on algorithm performance, we compared three different ResNet architectures in the Seq-Cifar-10 dataset. As shown in~\ref{tab:resnet_abla}, our algorithm effectively achieves great improvements in different architectures.

\textbf{Effectiveness of symmetric prototype-based loss.} 
To verify the effectiveness of our prototype-based loss, we 
compare two contrastive losses, including our symmetric loss~(\ref{apx: CacheNCE}) and the asymmetric instance-wise loss proposed by \cite{cha2021co2l}, all of which are carried out in Seq-Cifar-10 in the Class-IL setting without distillation loss. As demonstrated in Table \ref{tab:NCE_ablation}, symmetric prototype-based InfoNCE is better than asymmetric instance-wise InfoNCE, which indicates that our InfoNCE can obtain a high-quality representation space with prototypes. Interestingly, the symmetric instance-wise contrastive loss performs worse than the asymmetric one in continual learning scenarios in the past literature~\cite{cha2021co2l}. With the results demonstrated in Table~\ref{tab:NCE_ablation}, we can believe that prototype-based contrastive loss can preserve knowledge better than instance-wise contrastive loss due to the observation that all prototypes can be preserved while most instances of previous tasks are lost online.

\end{document}